\documentclass{article}
\usepackage{iclr_aims,times}

\usepackage{xcolor}
\usepackage{caption}
\usepackage{bbm}
\usepackage{amsmath}
\usepackage{comment}

\usepackage{graphicx}
\graphicspath{{images/}}
\usepackage{amsmath,amssymb,amsthm}

\theoremstyle{plain}
\newtheorem{proposition}{Proposition}

\usepackage{subcaption}
\usepackage{amsmath,amssymb,amsthm}
\usepackage{enumitem}

\usepackage[ruled,vlined]{algorithm2e}

\usepackage{tikz}
\usetikzlibrary{positioning,arrows.meta}

\usepackage{hyperref}
\usepackage{url}

\usepackage{dblfloatfix}
\usepackage{float}
\usepackage[section]{placeins}

\newcommand{\fsingle}{\mathrm{F}_1^{\text{single}}}
\newcommand{\funion}{\mathrm{F}_1^{\text{union}}}

\SetCommentSty{mycommfont}
\SetKwComment{tcp}{$\triangleright$\ }{ }

\title{Visualizing Coalition Formation: \\ From Hedonic Games to Image Segmentation}

\author{Pedro H. de Paula Fran\c{c}a \\
  UFRJ \\
  \texttt{pedro.franca@ufrj.br} 
  \And
  Lucas Lopes Felipe \\
  UFRJ \\
  \texttt{lucaslopesf2@gmail.com} 
  \And
  Daniel Sadoc Menasché \\
  UFRJ \\
  \texttt{sadoc@ic.ufrj.br}
}

\iclrfinalcopy
\begin{document}
\raggedbottom
\maketitle

\begin{abstract}
We propose image segmentation as a visual diagnostic testbed for coalition formation in hedonic games. Modeling pixels as agents on a graph, we study how a granularization parameter shapes equilibrium fragmentation and boundary structure.
On the Weizmann single-object benchmark, we relate multi-coalition equilibria to binary protocols by measuring whether the converged coalitions overlap with a foreground ground-truth. We observe transitions from cohesive to fragmented yet recoverable equilibria, and finally to intrinsic failure under excessive fragmentation.
Our core contribution links multi-agent systems with image segmentation by quantifying the impact of mechanism design parameters on equilibrium structures.
\end{abstract}

\section{Introduction}
\label{sec:intro}

\begin{figure}[b]
\centering
\setlength{\tabcolsep}{1pt}
\begin{tabular}{cccccc}
\includegraphics[width=0.3\textwidth,height=0.12\textheight,keepaspectratio]{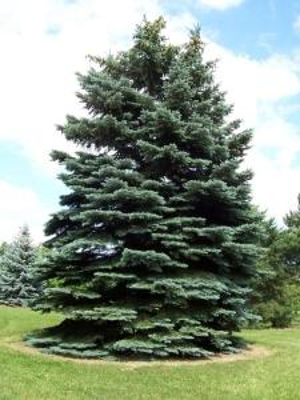} &
\includegraphics[width=0.3\linewidth,height=0.12\textheight,keepaspectratio]{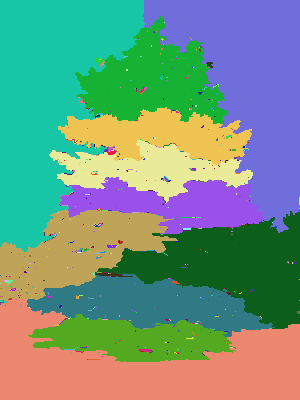} &
\includegraphics[width=0.3\linewidth,height=0.12\textheight,keepaspectratio]{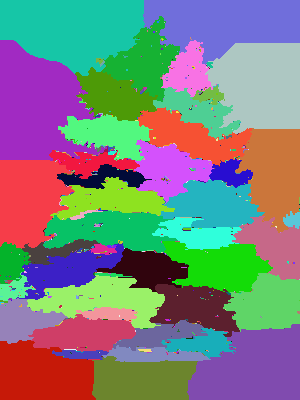} &
\includegraphics[width=0.3\linewidth,height=0.12\textheight,keepaspectratio]{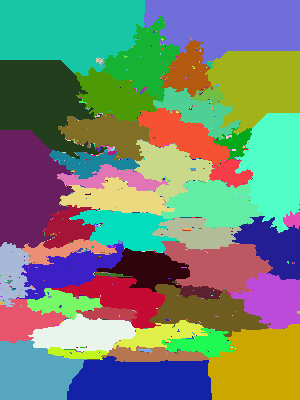} &
\includegraphics[width=0.3\linewidth,height=0.12\textheight,keepaspectratio]{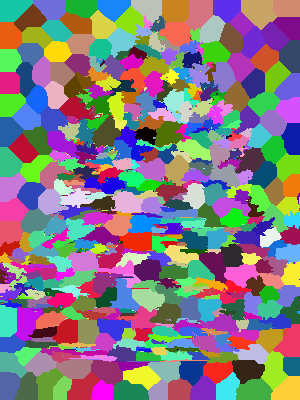} &
\includegraphics[width=0.3\linewidth,height=0.12\textheight,keepaspectratio]{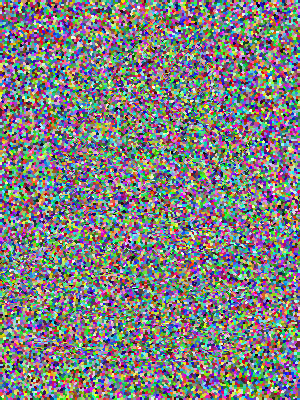} \\
(a) & (b) & (c) & (d) & (e) & (f)
\end{tabular}
\caption{Segmentation of the same image for increasing values of $\gamma$.
(a) Original image.
(b) $\gamma=10^{-6}$.
(c) $\gamma=8\times10^{-6}$.
(d) $\gamma=10^{-5}$.
(e) $\gamma=5\times10^{-4}$.
(f) $\gamma=2.5\times10^{-1}$.}% Small $\gamma$  yields few large coalitions, while larger $\gamma$ produces   more fragmented partitions.}
\label{fig:resolution_sequence}
\end{figure}

Coalition formation is a common emergent behavior in multi-agent systems and mechanism design: agents optimize
individual utilities under interaction constraints, and equilibria appear as partitions of the
population~\citep{bogomolnaia2002stability}. %\footnote{Such as: precise delineation of organs and tumors in medical imaging; monitoring land cover change and urban growth in environmental and earth sciences; distinguishing microscopic structures and phases in materials science; mapping anatomical regions in neuroscience brain scans; and enabling real-time object detection for safe navigation in autonomous vehicles and robotics.}
We leverage intuitive image segmentation to interpret a multi-agent coalition mechanism, providing a quantitative measure of a resolution parameter that controls coalition granularity. 
Image segmentation is an essential task across a range of scientific and technological domains~\citep{mittal2022comprehensive}.

The \emph{resolution} parameter $\gamma$ modulates coalition granularity~\citep{traag2011cpm}, fundamentally reshaping the equilibrium structure of the system~\citep{felipe2025leiden}.
By adjusting $\gamma$, the system can span the full spectrum of partitions, ranging from a single grand coalition as $\gamma \to 0$ to a set of individual singletons as $\gamma \to 1$. 
Figure~\ref{fig:resolution_sequence} 
illustrates how image segmentation visualizes this transition from smaller    to larger values of $\gamma$ (Fig.~\ref{fig:resolution_sequence}(b) to~Fig.~\ref{fig:resolution_sequence}(f)), mapping different partition equilibria to varying levels of spatial granularity.  %Small $\gamma$ (Fig.~\ref{fig:resolution_sequence}(b)) yields few large coalitions, while larger $\gamma$ produces   more fragmented partitions with finer boundaries. %  Starting from Fig.~\ref{fig:resolution_sequence}(a), the regimes  with  $\gamma=8 \times 10^{-6}$  and  $\gamma=2.5 \times 10^{-1}$ are illustrated in Figs.~\ref{fig:resolution_sequence}(b) and.  
Indeed, a central challenge for mechanism designers is identifying the specific value of $\gamma$ that yields meaningful equilibria for a given real-world application.

\begin{figure}[t]
\centering
\setlength{\tabcolsep}{2pt}
\begin{tabular}{ccccc}

\includegraphics[width=0.15\textwidth]{100_0497.png} &
\includegraphics[width=0.15\textwidth,trim=45 45 45 45,clip]{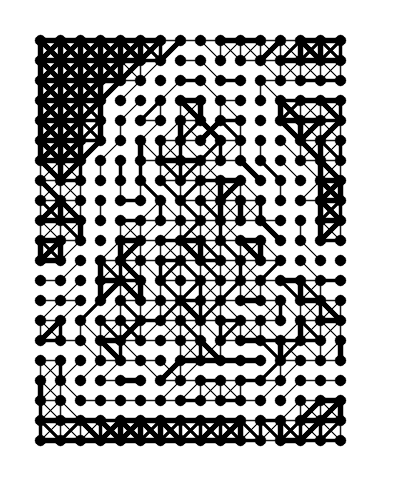} &
\includegraphics[width=0.15\textwidth,trim=45 45 45 45,clip]{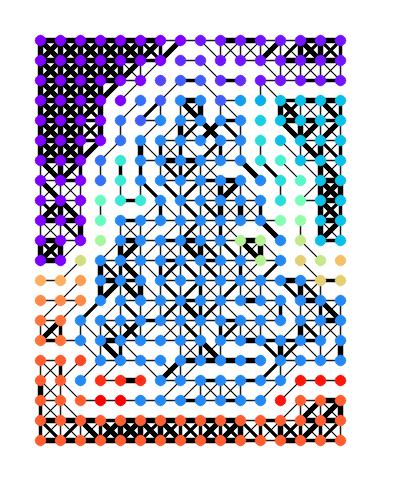} &
\includegraphics[width=0.15\textwidth]{100_0497_segmented_0.000001.png} &
\includegraphics[width=0.15\textwidth]{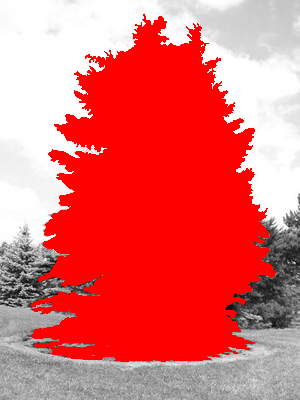} \\

(a) & (b) & (c) & (d) & (e)

\end{tabular}
\caption{Pipeline:
(a) image $\rightarrow$
(b) graph construction (downsampled,\textsuperscript{1} see Appendix~\ref{sec:graph}) $\rightarrow$
(c) equilibrium partition (Section~\ref{sec:mechanism}) $\rightarrow$
(d) segmented image $\rightarrow$
(e) ground truth evaluation (Section~\ref{sec:results}).}
\vspace{-0.1in}
\label{fig:pipeline_figs}
\end{figure}

We bridge image segmentation and the proposed coalition mechanism by representing the image as a graph where pixels act as nodes.
In this framework, the primary distinction between methods lies in the edge construction, which encodes pairwise similarities according to specific schemes~\citep{shi1997normalized, galun2003texture, AlpertGBB07}.
This formulation effectively casts the segmentation task as a network clustering or community detection problem~\citep{avrachenkov2018network}.

  Our key  contribution lies in using the pipeline in Figure~\ref{fig:pipeline_figs} to bridge the gap between   multi-agent systems and intuitive image segmentation tasks,  so as to quantify the impact of the resolution parameter on equilibrium structures. 
%Figure~\ref{fig:pipeline_figs} illustrates the stages of our processing pipeline. 
At the core of the pipeline lies the coalition mechanism (Figure~\ref{fig:pipeline_figs}(c)), which determines the specific cluster assignment for each pixel; these assignments then serve as labels to reconstruct the segmented image (Figure~\ref{fig:pipeline_figs}(d)). In Section~\ref{sec:mechanism}, we detail the mechanism based on hedonic games~\citep{felipe2025leiden}. Section~\ref{sec:results} then evaluates the overlap between the equilibrium partitions produced by the mechanism and the ground-truth foreground coalition (Figure~\ref{fig:pipeline_figs}(e)).
Section~\ref{sec:conclusion} concludes with a summary of our contributions and future work.\footnote{
We detail our graph construction in Appendix~\ref{sec:graph}. While it is impactful for performance, our mechanism is agnostic to the specific construction method as long as the input is a weighted graph; thus, evaluating alternatives remains future work.   Figs.~\ref{fig:pipeline_figs}(b) and~\ref{fig:pipeline_figs}(c) are downsampled versions  of Fig.~\ref{fig:pipeline_figs}(a)  for visual clarity.}

\section{Hedonic Mechanism for Coalition Formation}
\label{sec:mechanism}

A hedonic game is a naturally decentralized coalition formation game in which an agent's preference depends strictly on the composition of its own coalition~\citep{aziz2016hedonic}.
A potential game allows unilateral improvements by individual agents to be aligned with the maximization of a global potential function, where local dynamics can be interpreted as hill-climbing a single objective \citep{monderer1996potential}.

\cite{felipe2025leiden} modeled the Constant Potts Model (CPM) \citep{traag2011cpm}, a quality function known for being a \emph{resolution-limit-free} method~\citep{fortunato2007resolution}, as an \emph{additively separable potential hedonic game}. The CPM is decomposed into individual agent utilities (hedonic potentials). This frames the mechanism as a non-cooperative game while remaining aligned with a global quality function, ensuring that selfish agents' decisions yield stable, cooperative coalitions.

\textbf{Potential. }
\label{subsec:potential}
Each player seeks to belong to the community that maximizes its utility, in which better balances attraction to well-connected neighbors against a penalty for
joining large, weakly connected coalitions.
For a node $v$ and community  $C$, we define
\begin{equation}
\mathrm{Potential}_{v}^{\gamma}(C)
= (1-\gamma)\, d(v,C)
- \gamma\, \overline{d}(v,C)
\label{eq:potencial_node}
\end{equation}
where $d(v,C)$ denotes the degree of $v$ in community $C$, while $\overline{d}(v,C)$ is the number of non-neighbors in community $C$, therefore $|C| = 1+ d(v,C) + \overline{d}(v,C)$ is the number of nodes in $C$.
The parameter $\gamma \in [0,1]$ controls the trade-off between cohesion and coalition size:
small $\gamma$ favors larger cohesive regions, while larger $\gamma$ penalizes large communities and promotes fragmentation into smaller coalitions (Figure~\ref{fig:resolution_sequence}).

\begin{algorithm}[t]
\footnotesize
\DontPrintSemicolon
\LinesNumbered
\SetKwInOut{Input}{Input}\SetKwInOut{Output}{Output}

\Input{Weighted Graph $G=(V,E,w)$, resolution $\gamma$, partition $\pi^{(0)}$}
\Output{Equilibrium partition $\pi=\{C_1,\dots,C_K\}$}

\If{$\pi^{(0)}$ not provided}{
    Assign each node to its own community \label{line:init-singleton} \tcp*[r]{Initialize a singleton partition}
}
\textbf{change} $\leftarrow$ \texttt{true} \label{line:change-true} \tcp*[r]{Activate the main improvement loop}

\While{\textbf{change}}{ \label{line:while} \tcp*[r]{Iterate while at least one beneficial node move occurs}
    \textbf{change} $\leftarrow$ \texttt{false} \label{line:reset-change} \tcp*[r]{Clear the change flag}
    
    \ForEach{node $v \in V$}{ \label{line:for-each-node} \tcp*[r]{Scan all nodes (pixels)}
        $\sigma_v \leftarrow$ current community of $v$ \label{line:current-community} \tcp*[r]{Retrieve current community $\sigma_v$}
        
        $\sigma^\star \leftarrow \arg\max_C \mathrm{Potential}_{v}^{\gamma}(C)$ \label{line:best-community} \tcp*[r]{get community maximizing potential}
        
        \If{$\mathrm{Potential}_{v}^{\gamma}(\sigma^\star) > \mathrm{Potential}_{v}^{\gamma}(\sigma_v)$}{ \label{line:if-better} \tcp*[r]{Check for strict potential improvement}
            
            Move $v$ to $\sigma^\star$ \label{line:move} \tcp*[r]{Reassign the node's community}
            
            \textbf{change} $\leftarrow$ \texttt{true} \label{line:set-change} \tcp*[r]{Mark change occurred to force another pass}
        }
    }
}
\caption{Hedonic optimization: The algorithm terminates at a locally stable equilibrium.}
\label{alg:hedonic_pixels}
\end{algorithm}

\textbf{Equilibrium. }
\label{sec:equilibrium}
An \emph{equilibrium} is a stable partition $\pi=\{C_1,\dots,C_K\}$ where no agent has an incentive to change their community. This state is reached when every node is \emph{locally stable}, meaning its current community $C_{\sigma_v}$ provides a potential greater than or equal to any other community $C_k$:
\begin{equation}
    \text{Potential}_{v}^{\gamma}(C_{\sigma_v}) \geq \text{Potential}_{v}^{\gamma}(C_k) \quad \forall k \neq \sigma_v.
    \label{eq:potential_greater}
\end{equation}
A partition is considered a solution to the community detection problem if it satisfies two conditions:
\begin{itemize}
    \item \textbf{Internal Stability:} No node wants to leave its current community.
    \item \textbf{External Stability:} No node wants to join a different existing community.
\end{itemize}
In this state, each node $v$ is assigned to a community $\sigma_v$ that maximizes its individual utility $\sigma_{v} \in \arg\max_k \text{Potential}_{v}^{\gamma}(C_k)$.
Movement only occurs if a change offers a \textbf{strictly higher} utility. If multiple communities provide the same maximum utility, the node remains in its current community.

%\textbf{Hedonic Optimization. }
\label{subsec:optimization}
%
%
% We leverage the Leiden algorithm~\citep{traag2019louvain2leiden}.
Algorithm~\ref{alg:hedonic_pixels} outlines the hedonic game mechanism,\footnote{
Algorithm~\ref{alg:hedonic_pixels} presents a simplified version of the optimization process to illustrate the underlying mechanics. In practice, scanning all nodes (Line~\ref{line:for-each-node}) results in many ``ignorable checks'' that do not yield strict improvements. To achieve high performance, we utilize the
% \texttt{hedonic} Python library~\citep{felipe2025leiden} (\url{https://github.com/lucaslopes/hedonic}),
\texttt{community\_leiden} method~\citep{traag2019louvain2leiden} (of the \texttt{igraph} Python library) optimizing the CPM~\citep{traag2011cpm},
which serves as the core of our pipeline.}
which relies on a potential function bounded by $2|V|^2$. Given a rational resolution parameter $\gamma = b/\kappa$, every improving move increases the potential by at least $1/\kappa$, guaranteeing convergence in $O(\kappa \cdot |V|^2)$ steps~\citep{felipe2025leiden}. % This result is pseudo-polynomial with respect to $\log c$ when $c$ is provided as part of the input~\citep{felipe2025leiden}.
%however, for any fixed resolution where $c$ is a constant, the algorithm converges in $O(|V|^2)$ steps.

\section{Results}
\label{sec:results}

To quantify how $\gamma$ shapes equilibrium partitions and yields meaningful segmentations, we evaluate the converged partitions using two post-hoc accuracy metrics inspired by~\cite{AlpertGBB07}.  For each image, we select one  binary ground-truth (GT),     denoted by $Y$, where 
$Y\in\{0,1\}^{H\times W}$:
\begin{itemize}
    \item \textbf{$\fsingle$ (Dominant-Coalition Accuracy):} Evaluates  the $F_1$ score of the \emph{single} community $C_k \in \pi$ that maximizes the $F_1$ score relative to $Y$; %the binary ground-truth (GT) mask $Y\in\{0,1\}^{H\times W}$.
    \item \textbf{$\funion$ (Recoverable-Union Accuracy):} Evaluates the  $F_1$ score of the \emph{subset} of communities $S \subseteq \{ C_1, \ldots, C_K \}$ whose union maximizes the $F_1$ score relative to $Y$. %the binary ground-truth (GT) mask $Y\in\{0,1\}^{H\times W}$.
\end{itemize}
 $F_1$ is the harmonic mean of precision and recall~\citep{christen2023review}. Specifically, $\fsingle$ measures the emergence of a cohesive foreground, while $\funion$ assesses whether the foreground object, even if fragmented, remains ``recoverable'' from the equilibrium structure. Formal definitions, algorithmic details, and additional properties of $\fsingle$ and $\funion$ are provided in Appendix~\ref{app:f1}.

We use the \textit{Weizmann Segmentation Evaluation Database}~\citep{AlpertGBB07} to obtain images with known ground truth, which is used exclusively at evaluation time. Note that we restrict our analysis to the single-object subset (100 natural images, 3 human ground-truth masks per image) and leave the two-object subset for future work.
Experiments can be reproduced at: \url{https://github.com/henriquepedro1991/hedonic-games-to-image-segmentation}.

% We focus on a recoverable greedy union diagnostic. Alternative oracle unions are possible, but a systematic comparison is left for future work.

\textbf{Resolution Design. }
\label{subsec:resolution}
 We extend the approach of \cite{avrachenkov2018network}---setting the resolution as the graph's edge density---by proposing a scaling approach to evaluate varying density ratios:
$$
\gamma = \frac{\text{density}(G)}{c}, \qquad \text{density}(G)=\frac{2|E|}{|V|(|V|-1)},
$$
where $c$ is a fixed constant. This normalization allows $\gamma$ to act as a consistent local decision threshold across graphs of varying sparsity. 

Figure~\ref{fig:results_main}   highlights the key insights from our experiments (see    Appendix~\ref{sec:report} for additional details).  
Our numerical results identified $c=900$ as the optimal value to place most instances in the fragmented-but-recoverable regime (Figure~\ref{fig:results_main}(a)).   
Across all images, we obtain $\mathbb{E}[\funion]\approx 0.828$ (median $\approx 0.868$) and
$\mathbb{E}[\fsingle]\approx 0.488$, with an average gap
$\mathbb{E}[\funion-\fsingle]\approx 0.340$ (cf.  Figure~\ref{fig:results_main}(b)). 
The significant average performance gap   suggests that many apparent segmentation failures are actually fragmented, yet entirely recoverable, equilibria.  Indeed, the partition produced by the hedonic mechanism often contains the object, but distributed across several coalitions.

% Again, this large systematic gap shows that many apparent failures under a dominant-coalition criterion  arise from the difficulty of recomposing fragmented equilibria into a single binary mask.

\begin{figure}[t]
\centering
\setlength{\tabcolsep}{2pt}
\begin{tabular}{ccccc}
\includegraphics[width=0.3\textwidth,height=0.12\textheight,keepaspectratio]{100_0497.png} &
\includegraphics[width=0.3\textwidth,height=0.12\textheight,keepaspectratio]{100_0497_7.png} &
\includegraphics[width=0.3\linewidth,height=0.12\textheight,keepaspectratio]{100_0497_segmented_0.000001.png} &
\includegraphics[width=0.3\textwidth,height=0.12\textheight,keepaspectratio]{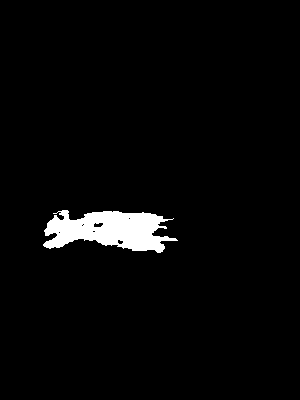} &
\includegraphics[width=0.3\textwidth,height=0.12\textheight,keepaspectratio]{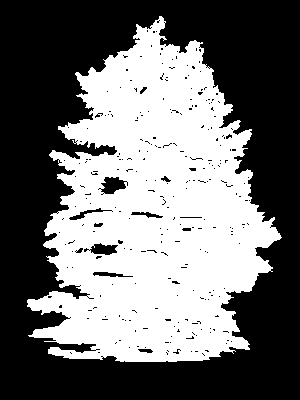} \\
(a) & (b) & (c) & (d) & (e)
\end{tabular}
\caption{Projections from a multi-community partition.
(a) Original image.
(b) Binary ground-truth mask.
(c) Equilibrium partition by hedonic mechanism.
(d) Best single community selected by $\fsingle$ at $\gamma=7.63\times10^{-6}$.
(e) Best subset of communities selected by $\funion$ at $\gamma=2.96\times10^{-5}$.}
\label{fig:binary_projections}
\end{figure}

\begin{figure}[t]
\centering
\setlength{\tabcolsep}{3pt}
\begin{tabular}{cc}
\includegraphics[width=0.5\textwidth,height=0.18\textheight,keepaspectratio]{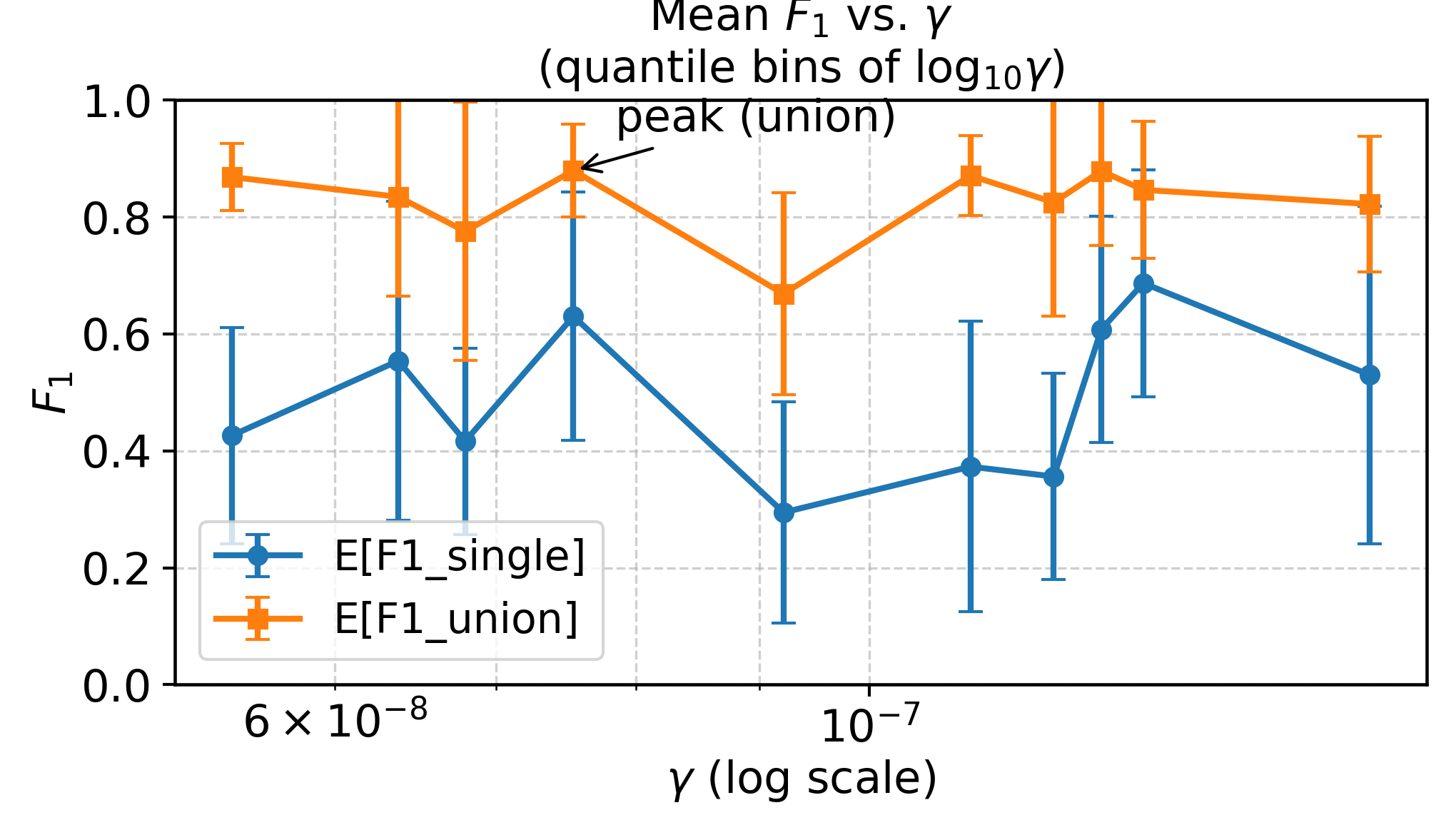} &
\includegraphics[width=0.5\textwidth,height=0.18\textheight,keepaspectratio]{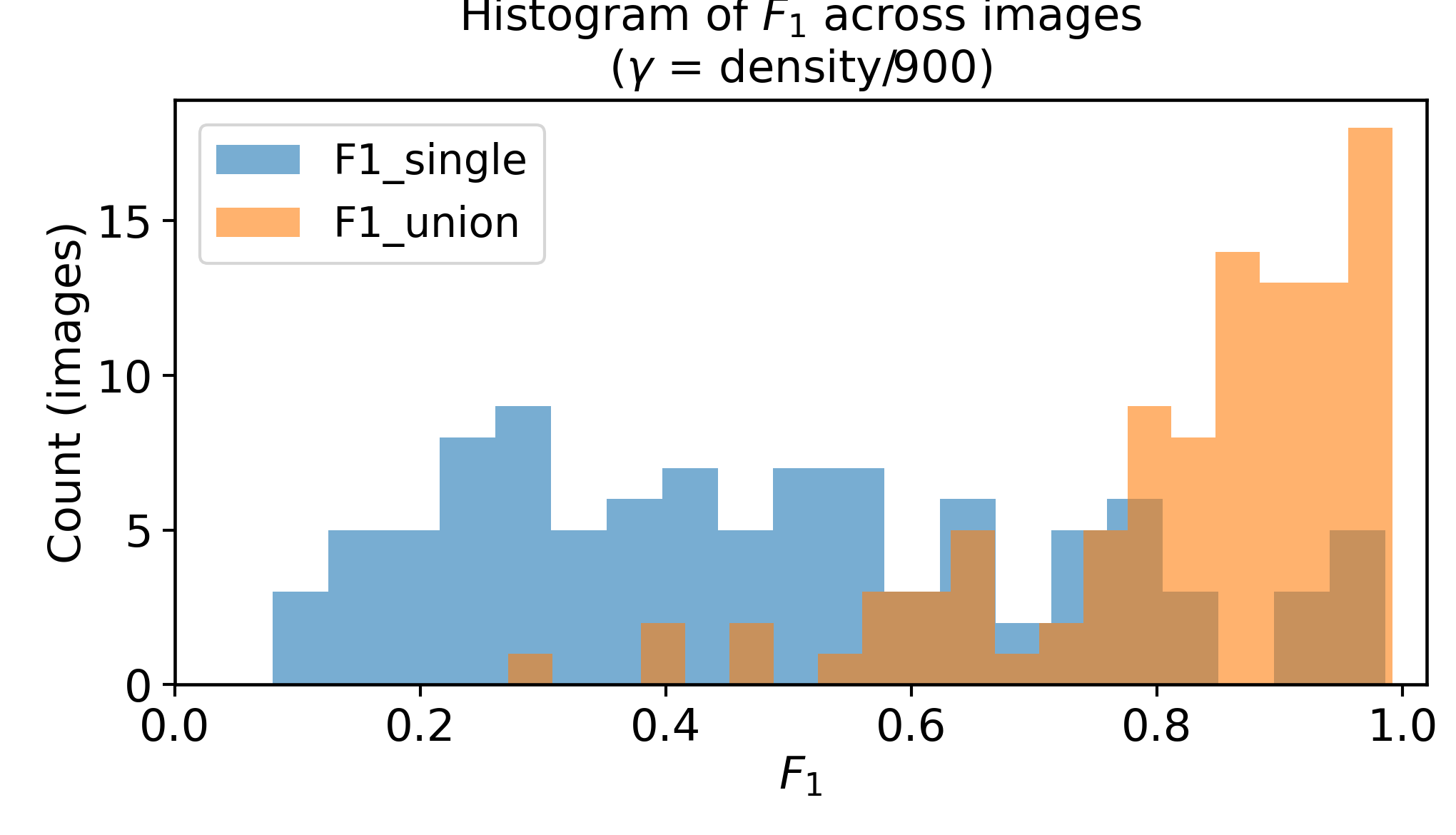} \\ 
(a) & (b) 
\end{tabular} \vspace{-0.15in}
\caption{
(a) Mean $\fsingle$ and $\funion$ as a function of $\gamma$
(induced by $\mathrm{density}(G)/900$).
(b) Global distributions of $\fsingle$ and $\funion$ over 100 images after selecting the best GT
per image.
}
\label{fig:results_main} \vspace{-0.1in}
\end{figure}

\section{Conclusion}
\label{sec:conclusion}

We proposed image segmentation as an interpretable testbed for studying equilibrium in hedonic games. This spatial and visual setting makes equilibrium structures directly \emph{inspectable}, providing an intuitive way to analyze how coalition mechanisms behave. In particular, it highlights how a resolution parameter reshapes equilibrium geometry in a visually grounded domain.

Our results demonstrate that a density-normalized rule, $\gamma=\mathrm{density}(G)/c$, places most instances in favorable regimes by successfully preventing extreme over- or under-fragmentation (Appendix~\ref{app:fragmentation}).
Appendix~\ref{app:peak_decay} reveals that a low $\funion-\fsingle$ gap indicates either mutual success for cohesive structures or intrinsic failure for detailed objects.
Appendix~\ref{app:init} confirms the robustness of the mechanism, demonstrating no difference in equilibrium accuracy whether initializing as a fragmenting grand coalition or as merging isolated nodes.
Finally, Appendix~\ref{app:f1} formalizes the computation of these binary projections and establishes the theoretical recoverability limits of $\funion$.

Future work includes evaluating alternative graph constructions and the Weizmann two-object subset. We also aim to investigate the relationship between accuracy and resolution, focusing on merging fragmented but recoverable coalitions to reduce the gap between $\fsingle$ and $\funion$.

{\footnotesize
\setlength{\bibsep}{0pt}
\bibliographystyle{plainnat}
\bibliography{main}
}

\appendix

% \clearpage
\section{Pixel-Graph Construction}
\label{sec:graph}

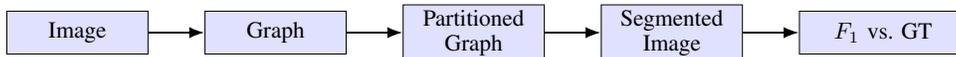
\begin{figure}[b]
\centering
\resizebox{0.92\linewidth}{!}{
\begin{tikzpicture}[
  node distance=0.85cm,
  every node/.style={draw, fill=blue!12, rectangle, minimum height=0.65cm, text width=2.0cm, align=center, inner sep=2pt},
  arrow/.style={-{Latex[scale=1.0]}, thick}
]
  \node (img) {Image};
  \node (g)   [right=of img] {Graph};
  \node (pg)  [right=of g] {Partitioned\\Graph};
  \node (si)  [right=of pg] {Segmented\\Image};
  \node (ev)  [right=of si, text width=2.4cm] {$F_1$ vs.\ GT};

  \draw [arrow] (img) -- (g);
  \draw [arrow] (g) -- (pg);
  \draw [arrow] (pg) -- (si);
  \draw [arrow] (si) -- (ev);
\end{tikzpicture}}
\caption{Diagnostic pipeline: image $\rightarrow$ graph $\rightarrow$ hedonic partition
$\rightarrow$ binary projection and evaluation.}
\label{fig:pipeline}
\end{figure}

Recall that our pipeline is given by Figure~\ref{fig:pipeline}. Next, we   focus on the pixel-graph construction (Image $\rightarrow$ Graph). 
Each image is converted into an undirected weighted pixel graph
$G=(V,E,w)$, where each pixel corresponds to a node and candidate edges
connect spatially adjacent pixels under an 8-neighborhood system,
including horizontal, vertical, and diagonal neighbors.

Rather than using a binary adjacency rule based on exact intensity agreement,
we assign each neighboring pair $(u,v)$ a continuous affinity weight that
combines color similarity and boundary evidence.
Let $I(u)\in[0,255]^3$ denote the RGB vector at pixel $u$, and let
$B:\Omega\to[0,1]$ be a normalized edge map used as a proxy for boundary
probability. For each adjacent pair $(u,v)$, we define
\begin{equation}
w_{uv}
=
\exp\!\left(
-\frac{\|I(u)-I(v)\|_2^2}{\sigma_{\mathrm{color}}^2}
\right)
\exp\!\left(
-\frac{\max\{B(u),B(v)\}}{\sigma_{\mathrm{edge}}^2}
\right).
\label{eq:weights}
\end{equation}

In our implementation, $B$ is obtained from a normalized Canny edge map,
used as a lightweight proxy for contour strength rather than the original
gPb detector. Thus, affinities are high when adjacent pixels have similar
RGB values and low when they lie near strong image boundaries.
Edges with very small affinity are discarded, yielding a sparse local graph.

This construction biases the graph toward forming coalitions inside
locally homogeneous color regions while reducing affinity across strong
visual discontinuities. Consequently, the subsequent community-detection
hedonic game is encouraged to group pixels within coherent regions and to
place coalition boundaries near salient contours.

Although this construction is inspired by pairwise affinity ideas used in
graph-based segmentation, it is not a full CPMC~\citep{carreira2011cpmc} formulation:
unlike CPMC, we do not solve seeded constrained parametric min-cut problems.
Instead, we use the resulting weighted pixel graph as input to the
hedonic optimization described in Section~\ref{subsec:optimization}.

Note that for the qualitative visualizations in Fig.~\ref{fig:pipeline_figs}(b)-(c), 
the graphs were constructed using a downsampled version of the original image 
to ensure the nodes and edges remain discernible for the reader. 
However, all quantitative experiments and results reported in 
Section~\ref{sec:results} were conducted using the original 
pixel-level resolution of the Weizmann dataset without any 
pre-processing or scaling.

\section{Numerical Report}
\label{sec:report}

\begin{comment}
\begin{figure}[b]
\centering
\setlength{\tabcolsep}{3pt}
\begin{tabular}{cc}
\includegraphics[width=0.5\textwidth,height=0.18\textheight,keepaspectratio]{f1_vs_gamma_mean.png} &
\includegraphics[width=0.5\textwidth,height=0.18\textheight,keepaspectratio]{f1_hist_single_union.png} \\ 
(a) & (b) 
\end{tabular} \vspace{-0.15in}
\caption{
(a) Mean $\fsingle$ and $\funion$ as a function of $\gamma$
(induced by $\mathrm{density}(G)/900$).
(b) Global distributions of $\fsingle$ and $\funion$ over 100 images after selecting the best GT
per image.
}
\label{fig:results_main}
\end{figure}
\end{comment}

We evaluate the proposed diagnostic framework on the Weizmann single-object subset, focusing on how
the resolution parameter $\gamma$ shapes equilibrium structure in hedonic coalition formation and
how this structure is reflected by the complementary scores $\fsingle$ and $\funion$. Our goal is  to characterize equilibrium regimes and failure modes
of the mechanism.

\textbf{Protocol and Ground-Truth Selection.}
Each image has three human ground-truth (GT) masks. For each GT we compute $\fsingle$ and
$\funion$, and then select, per image, the GT that maximizes $\funion$. We found low sensitivity to the choice of ground truth, as reported in Appendix~\ref{app:init}.
%This choice is consistent with the single-object assumption and ensures that $\funion$ reflects the best achievable recoverability of the foreground from the partition.  
Unless otherwise stated, we use the density-normalized resolution
$\gamma = \mathrm{density}(G)/900$. 
Figure~\ref{fig:results_main}(a) reports the mean $\fsingle$ and $\funion$ over 100 images as a
function of $\gamma$, while Figure~\ref{fig:results_main}(b) shows their global distributions.

 \textbf{Regime Transitions.}
As shown in Figure~\ref{fig:results_main}(a), for $\gamma < 10^{-7}$,
$\fsingle$ and $\funion$ are close on average, indicating a \emph{cohesive regime} in which the
foreground typically appears as a single dominant coalition.
As $\gamma$ increases, the two curves separate: $\fsingle$ decreases markedly, while $\funion$
remains comparatively high over a broad interval.
This gap characterizes a \emph{fragmented but recoverable regime}, where the foreground is present
in the partition but distributed across multiple coalitions.
 For larger $\gamma$, the behavior becomes more heterogeneous across images:
$\funion$ often remains high, while $\fsingle$ becomes
non-monotone, showing a partial rebound followed by a mild decline.
This pattern is consistent with highly fragmented partitions in which the foreground can still be
recovered by aggregating multiple coalitions, even though it rarely emerges as a single dominant one. Operationally, the pair $(\fsingle,\funion)$ distinguishes three broad situations: cohesive success (both high), recoverable fragmentation (large gap), and intrinsic failure (both low).

\textbf{Average Performance. } 
Across all images, we obtain $\mathbb{E}[\funion]\approx 0.828$ (median $\approx 0.868$) and
$\mathbb{E}[\fsingle]\approx 0.488$, with an average gap
$\mathbb{E}[\funion-\fsingle]\approx 0.340$ (cf.  Figure~\ref{fig:results_main}(b)). This large systematic gap shows that many apparent
failures under a dominant-coalition criterion  arise
from the difficulty of recomposing fragmented equilibria into a single binary mask. Qualitative extremes further clarify this interpretation: Appendix~\ref{app:peak_decay} contrasts a peak case, where both projections closely match the object, with a decay case, where even recoverable union cannot recover a good foreground mask.

\textbf{Robustness to Initialization.}
We repeat the protocol with two extreme initializations: \emph{singleton} (each node starts alone)
and \emph{one-coalition} (all nodes start together), and observe negligible differences between
them (Appendix~\ref{app:init}). We also analyze fragmentation via the number of communities $K$:
$K$ increases monotonically with $\gamma$, and high $\fsingle$ occurs mainly at small $K$, whereas
high $\funion$ appears across a wide range of $K$. Additional diagnostics are provided in
Appendix~\ref{app:fragmentation}. These regimes also suggest distinct failure modes. A large $\funion-\fsingle$ gap indicates recoverable fragmentation, where the object is distributed across multiple coalitions. Cases where both scores are low are more consistent with intrinsic failure or background leakage, in which the partition itself does not support accurate figure--ground recovery even under recoverable recomposition.

\section{Fragmentation Diagnostics}
\label{app:fragmentation}

This appendix provides detailed diagnostics relating fragmentation level, resolution, and
performance. 
We begin by examining how the number of communities $K$ produced by the hedonic optimization
varies with the resolution parameter $\gamma$, as well as how overall $F_1$ scores depend on
$\gamma$.

Figure~\ref{fig:diagnostics_gamma}(a) shows $K$ versus $\gamma$. While the scatter exhibits
substantial dispersion,  there is a   positive upward trend: larger values of $\gamma$
are associated, on average, with larger numbers of communities, indicating increasingly
fragmented equilibrium partitions.

Figure~\ref{fig:diagnostics_gamma}(b) shows $F_1$ versus $\gamma$. The plot suggests that
$\funion$ tends to attain higher and more stable values across the considered range of
$\gamma$, whereas $\fsingle$ is more dispersed and harder to predict. In particular,
$\fsingle$ typically takes lower values for intermediate values of $\gamma$, consistent with
a fragmentation regime in which the foreground is split across multiple communities.

\begin{figure}[h]
\centering
\setlength{\tabcolsep}{6pt}
\begin{tabular}{cc}
\includegraphics[width=0.49\textwidth,height=0.25\textheight,keepaspectratio]{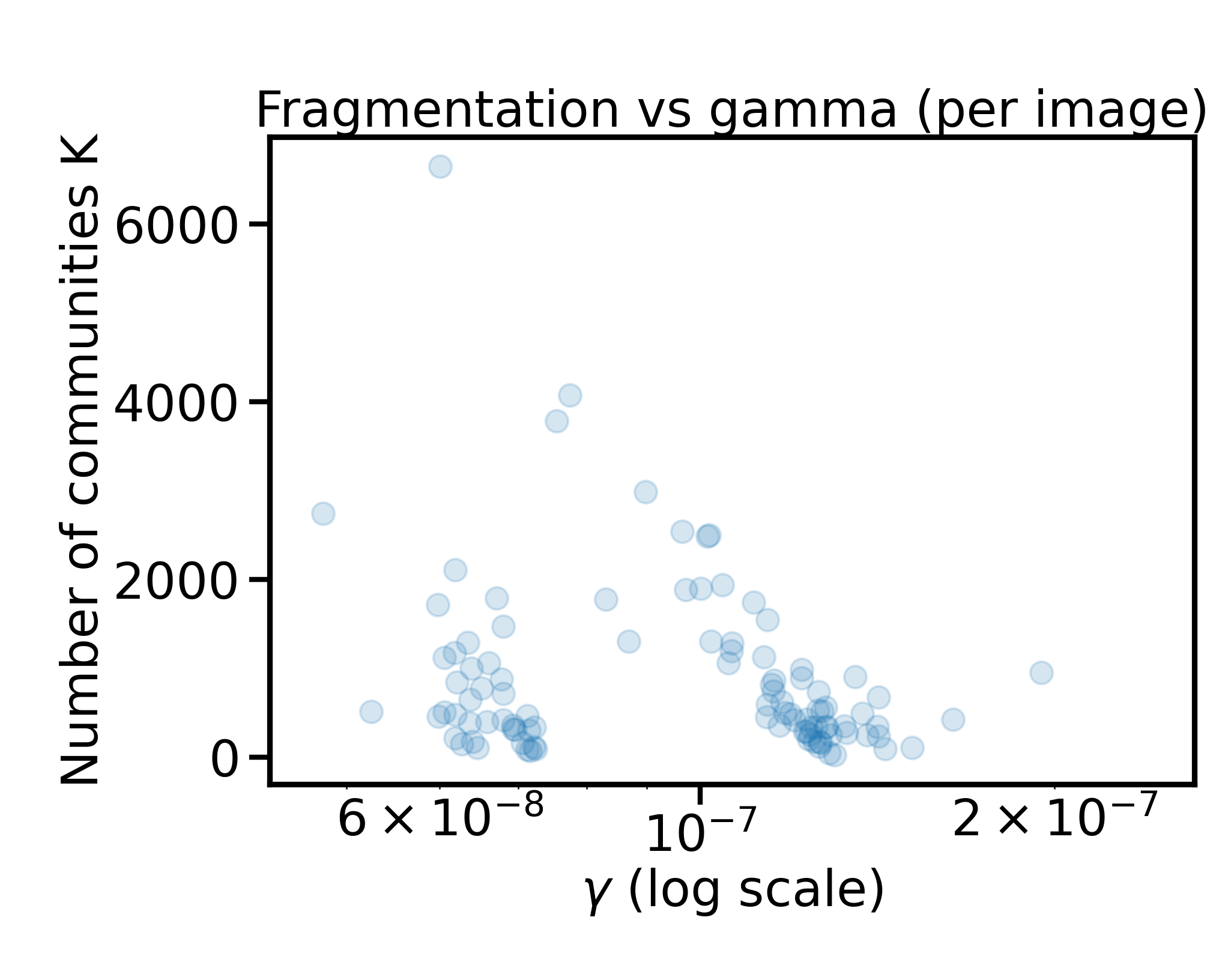} &
\includegraphics[width=0.52\textwidth,height=0.25\textheight,keepaspectratio]{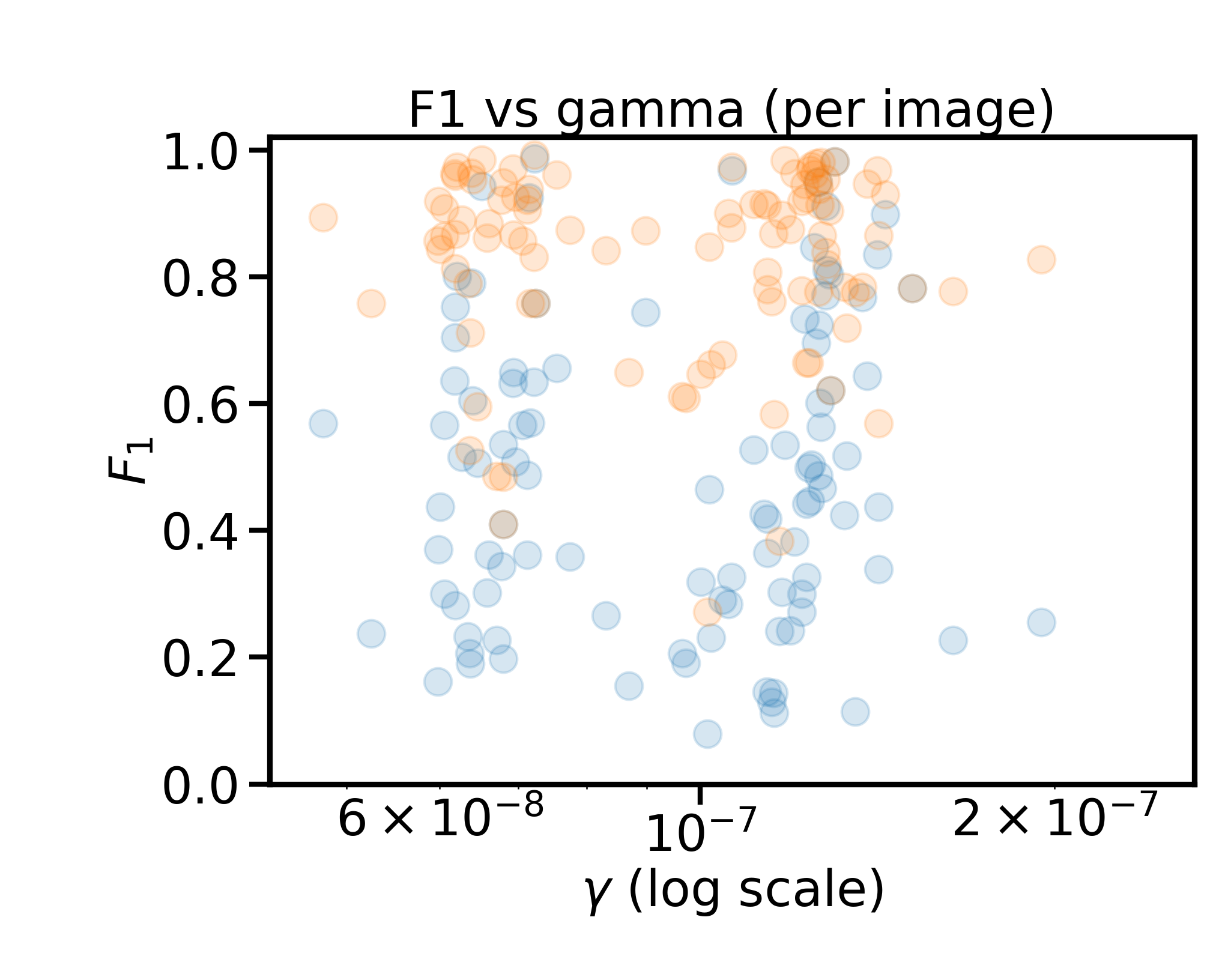} \\
(a) & (b)
\end{tabular}
\caption{Diagnostics as a function of the resolution parameter $\gamma$.
(a) Fragmentation $K$ vs.\ $\gamma$.
(b) $F_1$ vs.\ $\gamma$. Orange points correspond to $F_{1,\textrm{union}}$ and blue points to $F_{1,\textrm{single}}$.}
\label{fig:diagnostics_gamma}
\end{figure}

Next, we relate fragmentation directly to performance.
Figure~\ref{fig:f1_vs_K} shows $F_1$ versus the number of communities $K$. Values of $\fsingle$
exhibit substantial dispersion and are strongly concentrated at small $K$, suggesting that
dominant-coalition recovery typically requires low fragmentation and becomes unreliable as
fragmentation increases. In contrast, high $\funion$ values are observed across a wider range of $K$, indicating that the foreground can remain recoverable even when distributed
across many communities.

\begin{figure}[h!]
\centering
\includegraphics[width=0.6\textwidth,height=0.25\textheight,keepaspectratio]{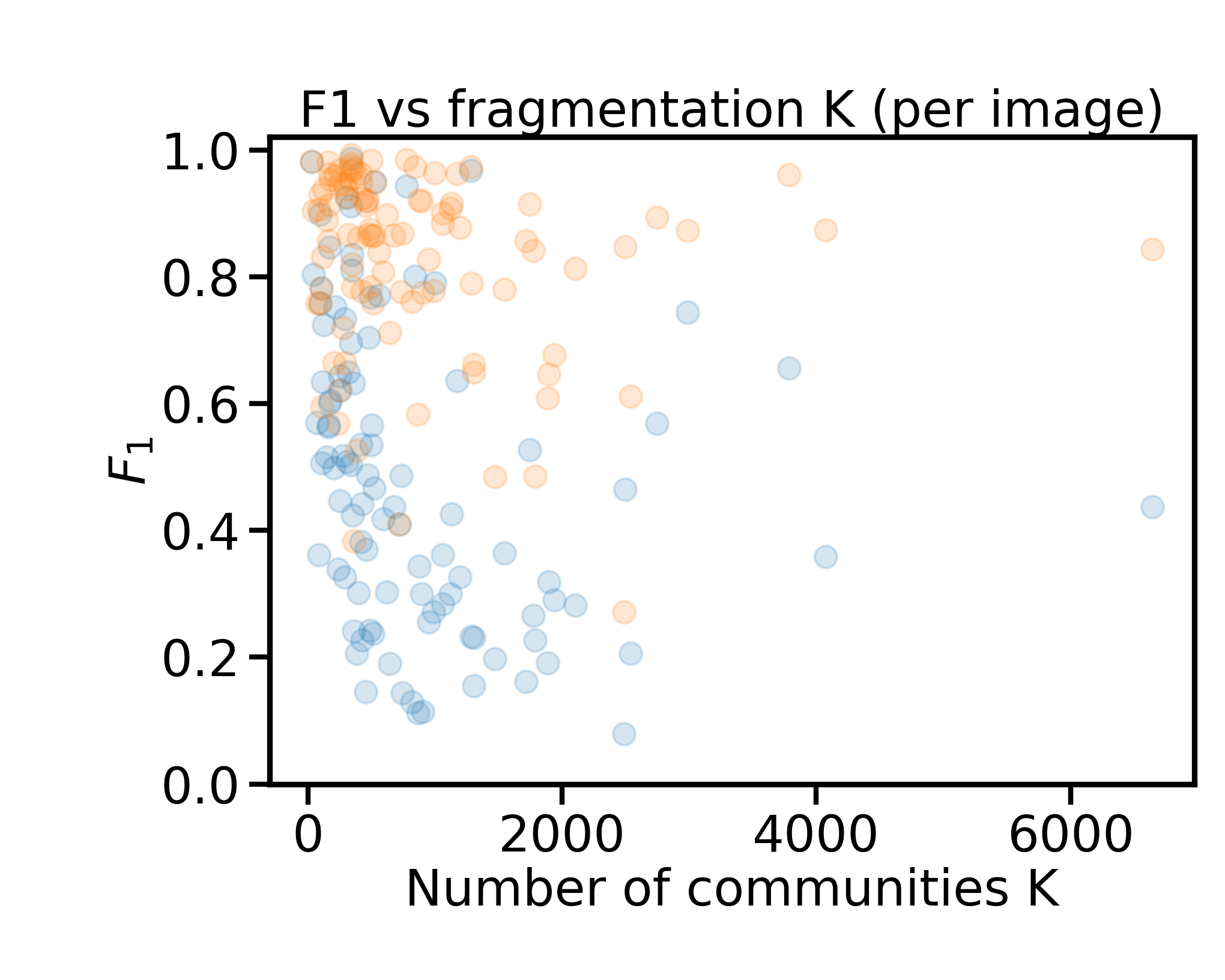}
\caption{$F_1$ vs \# of communities $K$. Orange and blue points correspond to $F_{1,\textrm{union}}$ and  $F_{1,\textrm{single}}$.} 
\label{fig:f1_vs_K}
\end{figure}

Finally, Figure~\ref{fig:K_hist} shows the histogram of $K$ over all images. The distribution
suggests that the density-normalized choice $\gamma=\mathrm{density}(G)/c$ places a substantial number of  instances
in an intermediate fragmentation regime, rather than at extreme values of very small or very
large $K$.

\begin{figure}[h!]
\centering
\includegraphics[width=0.6\textwidth,height=0.25\textheight,keepaspectratio]{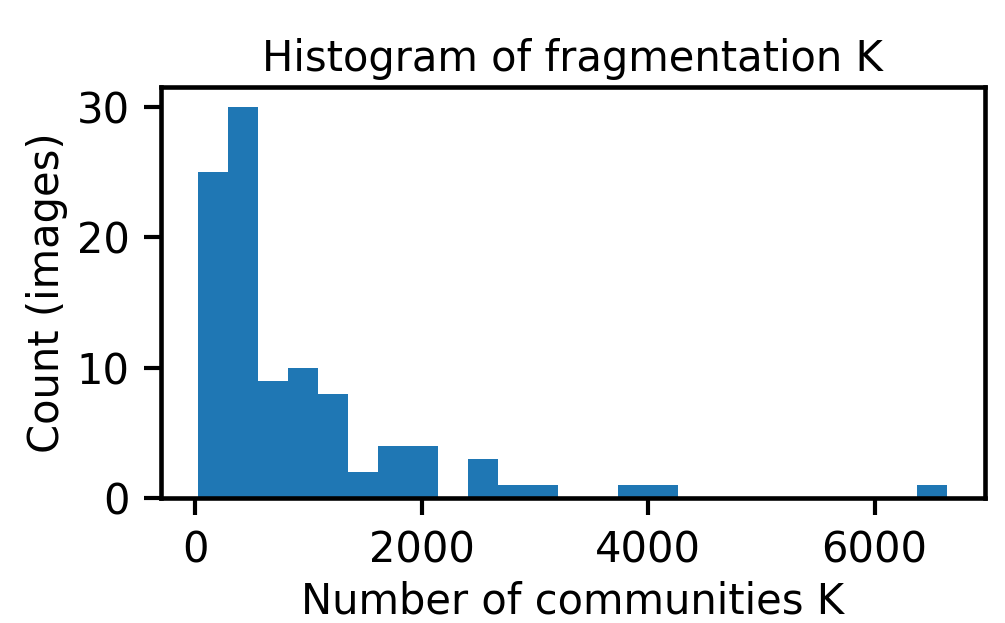}
\caption{Global distribution of the number of communities $K$ over 100 images.}
\label{fig:K_hist}
\end{figure}

% \clearpage
\section{Qualitative Extremes: Peak and Decay Cases}
\label{app:peak_decay}

To complement the aggregate trends reported in Section~\ref{sec:results}, we present two
qualitative extremes that help interpret how the scores $\fsingle$ and $\funion$ arise in
practice. We select one \emph{peak} case, corresponding to the image with the highest
$\funion$ in the dataset, and one \emph{decay} case, corresponding to the image with the
lowest $\funion$.

These examples clarify how the same multi-coalition mechanism can produce either highly
recoverable partitions or severe failure under fragmentation and/or background leakage.
In both cases, panels (a) and (b) show the original image and the selected ground-truth mask,
respectively; panel (c) shows the best single-community projection associated with $\fsingle$,
and panel (d) shows the recoverable union used to compute $\funion$.

\subsection{Peak case}
\label{subsec:peak}

Figure~\ref{fig:results_peak} shows a representative case in which the partition aligns very
well with the object. The best single coalition already captures most of the foreground,
yielding $\fsingle=0.9863$, and the recoverable union provides only a marginal improvement,
reaching $\funion=0.9922$. This is consistent with a cohesive regime in which the object
emerges almost entirely as a single dominant coalition.

\begin{figure}[h!]
\centering
\setlength{\tabcolsep}{2pt}
\renewcommand{\arraystretch}{0}
\begin{tabular}{cc}
\includegraphics[width=0.48\linewidth,height=0.145\textheight,keepaspectratio]{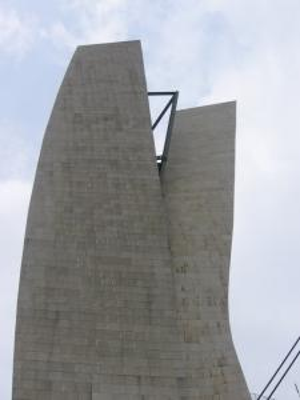} &
\includegraphics[width=0.48\linewidth,height=0.145\textheight,keepaspectratio]{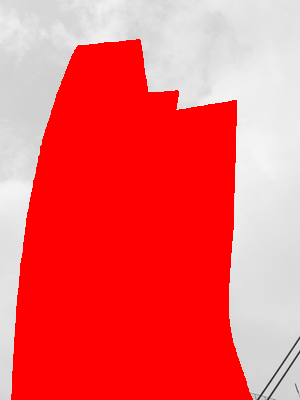} \\
\small (a) Original image &
\small (b) Selected ground truth \\
\includegraphics[width=0.48\linewidth,height=0.145\textheight,keepaspectratio]{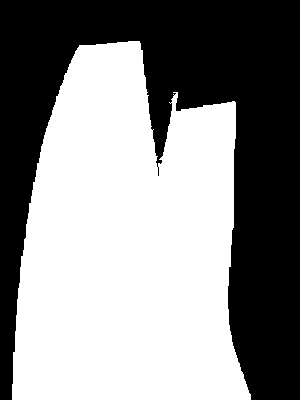} &
\includegraphics[width=0.48\linewidth,height=0.145\textheight,keepaspectratio]{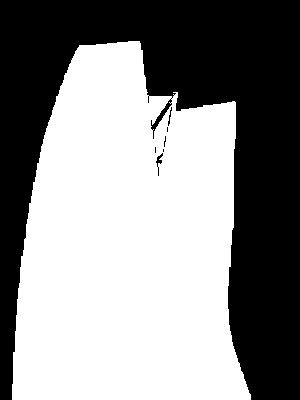} \\
\small (c) Best single coalition &
\small (d) Recoverable union
\end{tabular}
\caption{Peak example: the partition aligns closely with the selected ground truth.
The best single coalition already attains $\fsingle=0.9863$, while the recoverable union
slightly improves performance to $\funion=0.9922$.}
\label{fig:results_peak}
\end{figure}

\subsection{Decay case}
\label{subsec:decay}

Figure~\ref{fig:results_qual_decayed} shows the opposite extreme. Here the partition fails to
produce a coalition that matches the object well, giving $\fsingle=0.0370$. Even after
recoverable recomposition, the result remains poor, with $\funion=0.2714$. This case is
representative of intrinsic failure under severe fragmentation and/or figure--ground mixing:
the object is not cleanly represented in the partition, so even an optimistic oracle union
cannot recover a high-quality binary mask.

Taken together, these qualitative extremes reinforce the interpretation of the two-score
diagnostic. In peak cases, both $\fsingle$ and $\funion$ are high, indicating that the
foreground is well represented and largely cohesive in the partition. In decay cases, both
scores are low, indicating that the limitation is not merely downstream recomposition, but
rather the inability of the partition itself to represent the object cleanly under the chosen
graph construction and resolution regime.

\begin{figure}[h!]
\centering
\setlength{\tabcolsep}{2pt}
\renewcommand{\arraystretch}{0}
\begin{tabular}{cc}
\includegraphics[width=0.48\linewidth,height=0.145\textheight,keepaspectratio]{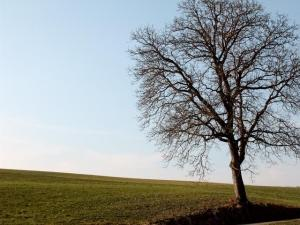} &
\includegraphics[width=0.48\linewidth,height=0.145\textheight,keepaspectratio]{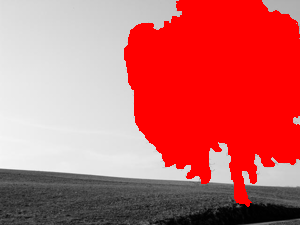} \\
\small (a) Original image &
\small (b) Selected ground truth \\
\includegraphics[width=0.48\linewidth,height=0.145\textheight,keepaspectratio]{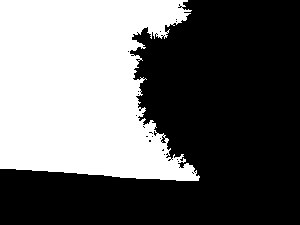} &
\includegraphics[width=0.48\linewidth,height=0.145\textheight,keepaspectratio]{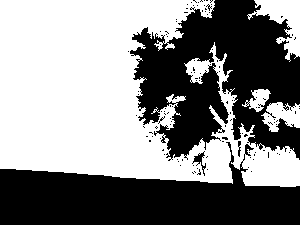} \\
\small (c) Best single coalition &
\small (d) Recoverable union
\end{tabular}
\caption{Decay example: the partition exhibits severe failure, with no single coalition
capturing the object well and only limited improvement under recoverable union.
The scores are $\fsingle=0.0370$ and $\funion=0.2714$.}
\label{fig:results_qual_decayed}
\end{figure}

\clearpage
\section{Sensitivity to Initialization and to Ground Truth Labels}
\label{app:init}

Figure~\ref{fig:results_violin_both} summarizes the $F_1$ distributions across all images for each ground truth under both initialization strategies.  Recall that the dataset provides three ground truth labels.  Each column of Figure~\ref{fig:results_violin_both}   corresponds to the use of a different label. Each plot reports the distribution over 100 images, with the top row corresponding to $\fsingle$ and the bottom row to $\funion$. Figure~\ref{fig:violin_singleton} shows the results obtained with singleton initialization  (each node starts alone), while Figure~\ref{fig:violin_onecoal} reports the one-coalition initialization (all nodes start together). We observe nearly identical distributions across both cases, indicating negligible sensitivity to initialization.

\begin{figure*}[h!]
\centering

\begin{subfigure}{0.70\textwidth}
  \centering
  \includegraphics[width=\textwidth]{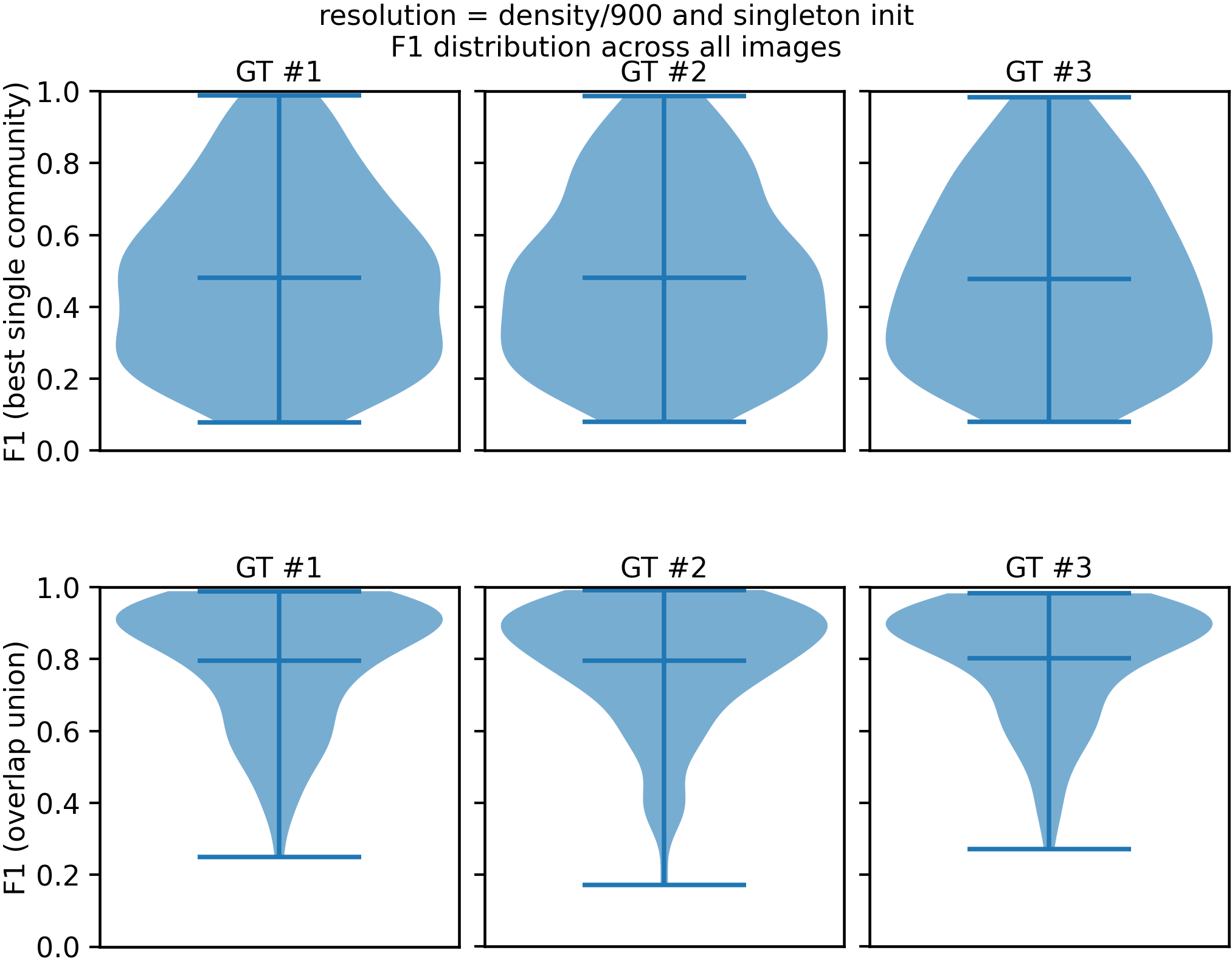}
  \caption{Singleton initialization.}
  \label{fig:violin_singleton}
\end{subfigure}

\vspace{4pt}

\begin{subfigure}{0.70\textwidth}
  \centering
  \includegraphics[width=\textwidth]{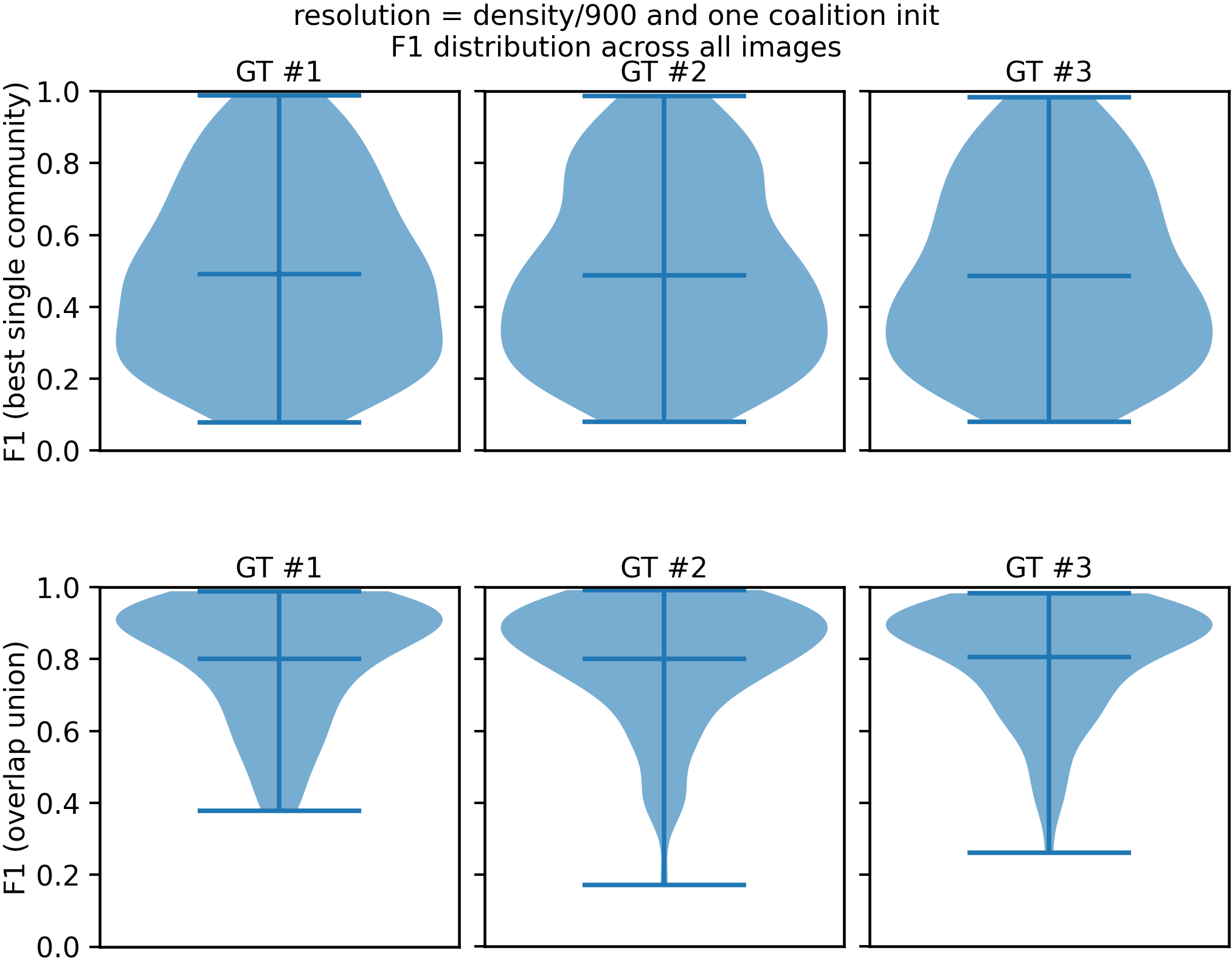}
  \caption{One-coalition initialization.}
  \label{fig:violin_onecoal}
\end{subfigure}

\caption{
Violin plots over 100 images with $\gamma=\mathrm{density}(G)/900$.
In each panel, the top row shows $\fsingle$ and the bottom row shows $\funion$.
}
\label{fig:results_violin_both}
\end{figure*}

% \clearpage
\section{Details of Binary Projections and $F_1$ Computation}
\label{app:f1}

This appendix provides a detailed mathematical specification of the binary projections
$\fsingle$ and $\funion$, as well as the underlying $F_1$ measure used throughout the paper.

\paragraph{\texorpdfstring{$\fsingle$}{F1-single} (dominant coalition).}
Let $\mathbf{L}\in\{1,\dots,K\}^{H\times W}$ be the label image. For each community label $k$,
define $\hat{Y}_k(p)=\mathbbm{1}[\mathbf{L}(p)=k]$, and set 
$\fsingle(\Pi)=\max_k F_1(\hat{Y}_k,Y)$.
This score measures whether the foreground object emerges as a \emph{single dominant coalition}.

\paragraph{\texorpdfstring{$\funion$}{F1-union} (recoverable union).}
Using the ground-truth mask \emph{only at evaluation time}, we construct
$\hat{Y}_S(p)=\bigvee_{k\in S}\mathbbm{1}[\mathbf{L}(p)=k]$ via a greedy forward pass that initializes
$S$ with the label attaining $\fsingle$, orders the remaining communities by their individual
$F_1(\hat{Y}_k,Y)$ scores, and then scans this list once, adding  communities while their  inclusion
increases the current $F_1$ (with a cap on merged labels). The resulting score
$\funion(\Pi)=F_1(\hat{Y}_S,Y)$ upper-bounds how well the foreground can be recovered from the
partition. %, regardless of whether it appears as a single coalition.

\subsection{Binary Masks and Pixel Sets}

Let the image domain be $\Omega=\{1,\dots,H\}\times\{1,\dots,W\}$.
The ground-truth foreground mask is
\[
Y(p) \in \{0,1\}, \qquad p \in \Omega,
\]
and the predicted community labeling is
\[
\mathbf{L}(p) \in \{1,\dots,K\}.
\]

For any binary mask $B:\Omega\to\{0,1\}$, define the corresponding foreground pixel set
\[
\mathcal{B}=\{p\in\Omega : B(p)=1\},
\qquad
\mathcal{Y}=\{p\in\Omega : Y(p)=1\}.
\]

\subsection{Precision, Recall, and $F_1$}

Given a predicted mask $B$, we define

\[
\mathrm{Precision}(B,Y)
=
\frac{|\mathcal{B}\cap \mathcal{Y}|}{|\mathcal{B}|},
\qquad
\mathrm{Recall}(B,Y)
=
\frac{|\mathcal{B}\cap \mathcal{Y}|}{|\mathcal{Y}|}.
\]

The $F_1$ score is the harmonic mean of precision and recall:

\[
F_1(B,Y)
=
\frac{2\,\mathrm{Precision}(B,Y)\,\mathrm{Recall}(B,Y)}
{\mathrm{Precision}(B,Y)+\mathrm{Recall}(B,Y)}
=
\frac{2|\mathcal{B}\cap \mathcal{Y}|}
{|\mathcal{B}| + |\mathcal{Y}|}.
\]

\subsection{$\fsingle$: Dominant-Coalition Projection}

For each community label $k$, we define the induced binary mask

\[
\hat{Y}_k(p)
=
\mathbbm{1}[\mathbf{L}(p)=k],
\qquad
\mathcal{C}_k=\{p : \mathbf{L}(p)=k\}.
\]

The dominant-coalition score is

\[
\fsingle(\Pi)
=
\max_{k\in\{1,\dots,K\}} F_1(\hat{Y}_k,Y)
=
\max_k
\frac{2|\mathcal{C}_k\cap \mathcal{Y}|}
{|\mathcal{C}_k| + |\mathcal{Y}|}.
\]

Thus, $\fsingle$ measures whether there exists a \emph{single} community whose pixels align well
with the foreground object.

\subsection{$\funion$: recoverable union}

For any subset of labels $S\subseteq\{1,\dots,K\}$, define the union mask

\[
\hat{Y}_S(p)
=
\bigvee_{k\in S}\mathbbm{1}[\mathbf{L}(p)=k],
\qquad
\mathcal{U}_S
=
\bigcup_{k\in S} \mathcal{C}_k.
\]

The corresponding $F_1$ score is

\[
F_1(\hat{Y}_S,Y)
=
\frac{2|\mathcal{U}_S\cap \mathcal{Y}|}
{|\mathcal{U}_S| + |\mathcal{Y}|}.
\]

We construct $S$ by a forward selection procedure that uses ground truth only at evaluation time:

\begin{enumerate}
\item Compute $F_1(\hat{Y}_k,Y)$ for all labels $k$ and sort labels in decreasing order of this score.
\item Initialize $S$ with the label attaining $F_{1,\textrm{single}}(\Pi)$.
\item Scan the sorted list once and add a label $k$ to $S$ while 
$F_1(\hat{Y}_{S\cup\{k\}},Y) > F_1(\hat{Y}_{S},Y)$.
\end{enumerate}

The resulting score is

\[
\funion(\Pi) = F_1(\hat{Y}_S,Y).
\]

%\subsection{Why $\funion$ is an Upper Bound}

$\funion$ approximates
\[
\max_{S\subseteq\{1,\dots,K\}} F_1(\hat{Y}_S,Y),
\]
which is the best achievable binary reconstruction from the given partition.
%Therefore, $\funion$ upper-bounds the recoverability of the foreground from the multi-coalition output.

\subsection{Greedy Forward Procedure for Computing $\funion$}

Algorithm~\ref{alg:greedy_union} summarizes the   greedy forward procedure used
throughout the paper to compute $\funion$. Starting from the best single community (the one
attaining $\fsingle$), the method scans the remaining communities in decreasing order of their
individual $F_1$ scores and adds a label only when its inclusion strictly improves the current
union score.

\begin{algorithm}[H]
\DontPrintSemicolon
\LinesNumbered
\SetKwInOut{Input}{Input}
\SetKwInOut{Output}{Output}

\Input{
Partition $\Pi=\{C_1,\dots,C_K\}$ (equivalently, label image $\mathbf{L}$)\;
Binary ground-truth mask $Y\in\{0,1\}^{H\times W}$\;
Optional cap $L_{\max}$ on the number of merged labels\;
}
\Output{
Selected label set $S$ and score $\funion(\Pi)=F_1(\hat{Y}_S,Y)$\;
}

Compute $F_1(\hat{Y}_k,Y)$ for all $k=1,\dots,K$\;
Let $k^\star \in \arg\max_k F_1(\hat{Y}_k,Y)$ and initialize $S \leftarrow \{k^\star\}$\;
Order the remaining labels $k\neq k^\star$ by decreasing $F_1(\hat{Y}_k,Y)$\;

\ForEach{label $k$ in the ordered list}{
    \If{$L_{\max}$ is specified and $|S| = L_{\max}$}{
        \textbf{break}\;
    }
    \If{$F_1(\hat{Y}_{S\cup\{k\}},Y) > F_1(\hat{Y}_S,Y)$}{
        $S \leftarrow S \cup \{k\}$\;
    }
}

\Return $S$ and $\funion(\Pi)=F_1(\hat{Y}_S,Y)$\;

\caption{Greedy forward union used to compute $\funion$.}
\label{alg:greedy_union}
\end{algorithm}

% \clearpage
\subsection{Alternative Recoverable-Union Variant}

For completeness, Algorithm~\ref{alg:overlap_union} presents a simple alternative  
union rule based on per-community overlap with the ground truth. Unlike the greedy forward
procedure used for $\funion$, this variant evaluates each community independently and selects all
labels whose individual overlap score exceeds a fixed threshold. We do not use this rule in the
main experiments; it is included only to illustrate that multiple oracle-style unions are
possible.

\begin{algorithm}[H]
\DontPrintSemicolon
\LinesNumbered
\SetKwInOut{Input}{Input}
\SetKwInOut{Output}{Output}

\Input{
Partition $\Pi=\{C_1,\dots,C_K\}$\;
Binary ground-truth mask $Y\in\{0,1\}^{H\times W}$\;
Overlap threshold $\tau \in [0,1]$\;
}
\Output{
Selected label set $S$ and union mask $\hat{Y}_S$\;
}

$S \leftarrow \emptyset$\;

\For{$k \in \{1,\dots,K\}$}{
    Compute the individual overlap score of $C_k$ with $Y$\;
    \If{$F_1(\hat{Y}_k,Y) > \tau$}{
        $S \leftarrow S \cup \{k\}$\;
    }
}

\Return $S$ and $\hat{Y}_S(p)=\bigvee_{k\in S}\mathbbm{1}[\mathbf{L}(p)=k]$\;

\caption{Threshold-based recoverable union as an alternative oracle variant.}
\label{alg:overlap_union}
\end{algorithm}

\subsection{Interpreting the Gap Between $\fsingle$ and $\funion$}

\begin{itemize}
\item Small gap: foreground appears largely as a single coalition.
\item Large gap: foreground is distributed across several communities (fragmentation).
\item Both small: intrinsic failure (object not represented well in partition).
\end{itemize}

Hence, the pair $(\fsingle,\funion)$ disentangles failures due to coalition structure from failures
due to downstream binary projection.

\subsection{Toy Example: Dominant Coalition vs.\ Recoverable Union}

Consider an image with $|\mathcal{Y}|=100$ foreground pixels and three communities
$C_1,C_2,C_3$ with the following \emph{spatial overlaps}:

\[
|\mathcal{C}_1|=40, \quad |\mathcal{C}_2|=35, \quad |\mathcal{C}_3|=50,
\]
\[
|\mathcal{C}_1\cap\mathcal{Y}|=30, \quad
|\mathcal{C}_2\cap\mathcal{Y}|=25, \quad
|\mathcal{C}_3\cap\mathcal{Y}|=10.
\]

\paragraph{Single-community scores.}
\[
F_1(C_1) = \frac{2\cdot 30}{40+100}=0.43,\qquad
F_1(C_2)=\frac{2\cdot25}{35+100}=0.37,\qquad
F_1(C_3)=\frac{2\cdot10}{50+100}=0.13.
\]

Thus,
\[
\fsingle = 0.43.
\]

\paragraph{Union of communities.}
Let $S=\{1,2\}$.
Then

\[
|\mathcal{U}_S|=75,
\qquad
|\mathcal{U}_S\cap\mathcal{Y}|=55.
\]

\[
F_1(\hat{Y}_S,Y)
=
\frac{2\cdot55}{75+100}
=0.63.
\]

Hence,
\[
\funion = 0.63 > \fsingle.
\]

This example illustrates a fragmented-but-recoverable situation: no single community captures the
object well, yet combining two communities yields good foreground reconstruction.

\subsection{Relation to Figure~\ref{fig:binary_projections}}

The behavior visualized in Figure~\ref{fig:binary_projections} corresponds to the label selection described in this section. In Figure~\ref{fig:binary_projections}(c), $\fsingle$ selects the single community with the highest individual $F_1$ score. 
 The $\funion$ mask is formed by the  greedy forward union described above, which aggregates communities whose addition improves the current $F_1$ score. This captures the entire object by aggregating the multiple fragments that individually correspond to parts of the foreground (e.g., $S=\{1,2\}$ in the toy example), without requiring an iterative optimization search.

\subsection{When the Recoverable-Union Can Be Suboptimal}

The recoverable-union procedure is a heuristic for approximating
\[
\max_{S \subseteq \{1,\dots,K\}} F_1(\hat{Y}_S, Y),
\]
which is combinatorial in $K$. As with many subset-selection methods, optimality is not
guaranteed.

A typical failure mode occurs when two communities are individually weak predictors of the
foreground but jointly complementary. For example, two communities may each cover disjoint halves
of the object while also containing substantial background. Individually, neither improves $F_1$
enough to be selected, but together they would yield a strong score.

Formally, the forward selection can fail when

\[
F_1(\hat{Y}_{\{i\}},Y) \le F_1(\hat{Y}_{S},Y)
\quad \text{and} \quad
F_1(\hat{Y}_{\{j\}},Y) \le F_1(\hat{Y}_{S},Y),
\]

but

\[
F_1(\hat{Y}_{S\cup\{i,j\}},Y) \gg F_1(\hat{Y}_{S},Y).
\]

Exact subset search would recover $\{i,j\}$, while forward selection would stop early.

In practice, this pathology is rare in our setting because foreground fragments tend to have nontrivial individual \emph{spatial overlap} with the object, making them attractive early additions.

\subsection{How Large Can the Gap $\funion - \fsingle$ Be?}

In the worst case, the gap between $\funion$ and $\fsingle$ can be arbitrarily large.

Consider an object split evenly across $M$ disjoint communities, each containing exactly
$\frac{1}{M}|\mathcal{Y}|$ foreground pixels and no background. Then, for each community $k$,

\[
F_1(\hat{Y}_{\{k\}},Y)
=
\frac{2\cdot \frac{1}{M}|\mathcal{Y}|}{\frac{1}{M}|\mathcal{Y}| + |\mathcal{Y}|}
=
\frac{2}{M+1}.
\]

Thus,

\[
\fsingle = \frac{2}{M+1}.
\]

If we union all $M$ communities, we recover the full object:

\[
F_1(\hat{Y}_{\{1,\dots,M\}},Y)=1,
\qquad
\funion = 1.
\]

Hence,

\[
\funion - \fsingle = 1 - \frac{2}{M+1},
\]

which approaches $1$ as $M \rightarrow \infty$.

This construction demonstrates that a very low $\fsingle$ does not necessarily indicate failure
of coalition formation: it may simply reflect extreme fragmentation of an otherwise perfectly
recoverable object.

In contrast, small gaps between $\fsingle$ and $\funion$ indicate that even multi-coalition
recomposition cannot recover the foreground, pointing to intrinsic failure modes such as
background leakage or poor boundary placement.

The pathological lower bound for $\fsingle$,  under perfect recoverability, is formalized through the following proposition:
\begin{proposition}[Arbitrarily small $\fsingle$ with $\funion=1$]
\label{prop:pathological-gap}
For any integer $M\ge 1$, there exists a partition $\Pi=\{C_1,\dots,C_M\}$ and a binary ground-truth
mask $Y$ such that
\[
\funion(\Pi)=1
\qquad \text{and} \qquad
\fsingle(\Pi)=\frac{2}{M+1}.
\]
Consequently,
\[
\funion(\Pi)-\fsingle(\Pi)=1-\frac{2}{M+1},
\]
which can be made arbitrarily close to $1$ by taking $M$ large.
\end{proposition}

\begin{proof}[Proof sketch]
Let $Y$ contain $|\mathcal{Y}|$ foreground pixels. Construct $M$ disjoint communities
$C_1,\dots,C_M$ such that each $C_k$ contains exactly $|\mathcal{Y}|/M$ foreground pixels and no
background pixels. Then for every $k$,
\[
\mathrm{TP}_k=\frac{|\mathcal{Y}|}{M},\qquad
\mathrm{FP}_k=0,\qquad
\mathrm{FN}_k=|\mathcal{Y}|-\frac{|\mathcal{Y}|}{M}.
\]
Substituting into the $F_1$ identity
$F_1= \frac{2\mathrm{TP}}{2\mathrm{TP}+\mathrm{FP}+\mathrm{FN}}$
gives $F_1(\hat{Y}_{\{k\}},Y)=\frac{2}{M+1}$, hence
$\fsingle(\Pi)=\frac{2}{M+1}$.
Moreover, the union $\hat{Y}_{\{1,\dots,M\}}$ equals $Y$, so $F_1(\hat{Y}_{\{1,\dots,M\}},Y)=1$.
Since $\funion$ is defined by a union-based projection, $\funion(\Pi)=1$.
\end{proof}

\end{document}